%% file: arxiv.tex
\pdfoutput=1
\documentclass{article}
\usepackage{times,authblk}

\usepackage{geometry}
\usepackage{rotating, color,subfigure,algorithm,algorithmic}

\usepackage{multirow}
\usepackage{multicol}
\usepackage{graphicx}
\usepackage{enumerate}
\usepackage{amsthm,amsmath,amssymb}
\usepackage{multicol}
\usepackage{caption}

\usepackage{mkolar_definitions}

\usepackage{ifthen}

\newcommand{\version}{arxiv}

\begin{document}
\title{Recovering Graph-Structured Activations using Adaptive Compressive Measurements}

\author[1]{
Akshay Krishnamurthy
\thanks{akshaykr@cs.cmu.edu}}
\author[2,3]{
James Sharpnack
\thanks{jsharpna@cs.cmu.edu}}
\author[2]{
Aarti Singh
\thanks{aarti@cs.cmu.edu}}

\affil[1]{Computer Science Department\\
Carnegie Mellon University}
\affil[2]{Machine Learning Department\\
Carnegie Mellon University}
\affil[3]{Statistics Department\\
Carnegie Mellon University}

\maketitle

\input{abstract.tex}

\input{intro.tex}

\input{results.tex}

\input{experiments.tex}
\input{conclusion.tex}

\appendix
\input{appendix.tex}

\bibliography{bibliography}
\bibliographystyle{plain}

\end{document}

%% file: abstract.tex
\begin{abstract}
We study the localization of a cluster of activated vertices in a graph, from
adaptively designed compressive measurements. We propose a hierarchical
partitioning of the graph that groups the activated vertices into few
partitions, so that a top-down sensing procedure can identify these partitions,
and hence the activations, using few measurements. By exploiting the cluster
structure, we are able to provide localization guarantees at weaker signal to
noise ratios than in the unstructured setting. We complement this performance
guarantee with an information-theoretic lower bound, providing a necessary
signal-to-noise ratio for any algorithm to successfully localize the cluster. We
verify our analysis with some simulations, demonstrating the practicality of our
algorithm.
\end{abstract}

%% file: intro.tex
\section{Introduction}
We are interested in recovering the support of a sparse vector $\xb \in
\mathbb{R}^n$ observed through the noisy linear model:
\[
y_i = a_i^T\xb + \epsilon_i
\]
Where $\epsilon_i \sim \mathcal{N}(0, \sigma^2)$ and $\sum_i ||a_i||^2 \le m$.  
This support recovery problem is well-known and fundamental to the theory of compressive sensing, which involves estimating a high-dimensional signal vector from few linear measurements~\cite{CandesCompressed}.
Indeed if $\xb$ is a $k$-sparse vector whose non-zero components are $\ge \mu$, it is now well known that one cannot identify these components if $\frac{\mu}{\sigma} = o(\sqrt{\frac{n}{m}\log (n/k)})$ and one can if $\frac{\mu}{\sigma} = \omega(\sqrt{\frac{n}{m}\log n})$, provided that $m \ge k \log n$~\cite{wainwright2009information}.
Indeed if $\xb$ is a $k$-sparse vector whose non-zero components are $\ge \mu$, it is now well known that one can identify these components if and only if $\frac{\mu}{\sigma} = o(\sqrt{\frac{n}{m}\log (n)})$ and $m \ge k \log(n/k)$~\cite{aeron2010information,wainwright2009information}.

We build upon the classical results of compressive sensing by developing
procedures that are \emph{adaptive} and that exploit additional \emph{structure}
in the underlying signal. Adaptivity allows the procedure to focus measurements
on activated components of the signal while structure can dramatically reduce
the combinatorial search space of the problem. Combined, both ideas can lead to significant
performance improvements over classical compressed sensing. This paper explores
the role of adaptivity and structure in a very general support recovery problem.

Active learning and adaptivity are not new ideas to the signal processing community and a number of papers in recent years have characterized the advantages and limits of adaptive sensing over passive approaches.
One of the first ideas in this direction was \emph{distilled sensing}~\cite{haupt2011distilled}, which uses direct rather than compressive measurements.
Inspired by that work, a number of authors have studied adaptivity in compressive sensing and shown similar performance gains~\cite{DavenportCastro,HauptNowak,MalloyNowak}.
These approaches do not incorporate any notion of structure.


The introduction of structure to the compressed sensing framework has also been explored by a number of authors~\cite{SoniHaupt,BaraniukCevher,balakrishnan2012recovering}.
Broadly speaking, these structural assumptions restrict the signal to a few of the ${n \choose k}$ linear subspaces that contain $k$-sparse signals.
With this restrictions, one can often design sensing procedures that focus on these allowed subspaces and enjoy significant performance improvements over unstructured problems.
We remark that both Soni and Haupt and Balakrishnan \emph{et. al.} develop adaptive sensing procedures for structured problems, but under a more restrictive setting than this study~\cite{SoniHaupt,balakrishnan2012recovering}.

This paper continues in both of these directions exploring the role of adaptivity {\em and} structure in recovering activated clusters in graphs.
We consider localizing activated clusters of nodes whose boundary in the graph is smaller than some parameter $\rho$.
This notion of structure is more general than previous studies, yet we are still able to demonstrate performance improvements over unstructured problems.

Our study of cluster identification is motivated by a number of applications in sensor networks measurement and monitoring, including identification of viruses in human or computer networks or contamination in a body of water.
In these settings, we expect the signal of interest to be localized, or clustered, in the underlying network and want to develop efficient procedures that exploit this cluster structure.

In this paper, we propose two related adaptive sensing procedures for identifying a cluster of activations in a network.
We give a sufficient condition on the signal-to-noise ration (SNR) under which the first procedure exactly identifies the cluster.
While this SNR is only slightly weaker than the SNR that is sufficient for unstructured problems, we show via information-theoretic arguments that one cannot hope for significantly better performance. 

For the second procedure, we perform a more refined analysis and show that the required SNR depends on how our algorithmic tool captures the cluster structure.
In some cases this can lead to consistent recovery at much weaker SNR.
The second procedure can also be adapted to recover a large fraction of the cluster.
We also explore the performance of our procedures via an empirical study.
Our results demonstrate the gains from exploiting both structure and adaptivity in support recovery problems. 
\ifthenelse{\equal{\version}{arxiv}}{}{Due to space restrictions, all proofs are available in an extended version of the paper~\cite{krishnamurthy2013arxiv}.}

\begin{table}
\centering
\begin{tabular}{|c | c | c|}
\hline
Setting & Necessary & Sufficient\\
\hline \hline
Passive, unstructured & $ \sqrt{\frac{n}{m} \log n}$~\cite{aeron2010information} & $
\sqrt{\frac{n}{m} \log n}$~\cite{aeron2010information}\\
Adaptive, unstructured & $ \sqrt{\frac{n}{m}}$~\cite{arias2011fundamental} & $
\sqrt{\frac{n}{m}\log k}$~\cite{HauptNowak}\\
Adaptive, structured & $ \sqrt{\frac{n}{m}}$ (Thm.~\ref{thm:lower})& $
\sqrt{\frac{n}{m} \log (\rho \log n)}$ (Prop.~\ref{prop:exact})\\
\hline
\end{tabular}
\caption{Compressed Sensing landscape.}
\label{tab:results}
\ifthenelse{\equal{\version}{arxiv}}{}{\vspace{-0.5cm}}
\end{table}

\begin{table}
\centering
\begin{tabular}{|c| c | c | c|}
\hline
Graph & Structure & Necessary & Sufficient\\
\hline \hline
2-d Lattice & Rectangle & $ \frac{1}{k} \sqrt{\frac{n}{m}}$~\cite{balakrishnan2012recovering} &
$\frac{1}{k} \sqrt{\frac{n}{m}}$~\cite{balakrishnan2012recovering}\\
Rooted Tree & Rooted subtree  & $\sqrt{\frac{k}{m}}$~\cite{soni2013fundamental} & $ \sqrt{\frac{k}{m} \log k}$~\cite{SoniHaupt}\\
Arbitrary & Best case & & $ \frac{1}{k} \sqrt{\frac{n}{m} \log((\rho +k) \log n)}$\\
\hline
\end{tabular}
\caption{Adaptive, Structured, Compressed Sensing.}
\label{tab:adaptive}
\ifthenelse{\equal{\version}{arxiv}}{}{\vspace{-0.5cm}}
\end{table}

We put our results in context of the compressed sensing landscape in Tables~\ref{tab:results} and~\ref{tab:adaptive}.
Here $k$ is the cluster size and, in the structured setting, $\rho$ denotes the number of edges leaving the cluster.
In the unstructured setting, Wainwright, and later Aeron \emph{et al.}, studied the passive support recovery problem while Haupt and Nowak consider the adaptive case~\cite{wainwright2009information,aeron2010information,HauptNowak}.
These works analyze algorithms with near-optimal performance guarantees.
Our work provides both upper and lower bounds for the adaptive structured setting.
Focusing on different notions of structure, Balakrishnan \emph{et. al.} give necessary and sufficient conditions for recovering a small square of activiations in a grid~\cite{balakrishnan2012recovering} while Soni and Haupt analyze the recovery of tree-sparse signals~\cite{SoniHaupt, soni2013fundamental}. 
Our work provides guarantees that depends on how well the signal is captured by our algorithmic construction.
In the worst case, we guarantee exact recover with an SNR of $\sqrt{\frac{n}{m} \log(\rho \log n)}$ (Proposition~\ref{prop:exact}) and in the best case, we can tolerate an SNR of $\frac{1}{k} \sqrt{\frac{n}{m} \log  ((\rho + k)\log n)}$ (Theorem~\ref{thm:approx}). 
It is worth mentioning that~\cite{SoniHaupt} obtains better results than ours, but study a very specific setting where the graph is a rooted tree and the signal is rooted subtree.

%% file: results.tex
\section{Main Results}
Let $C^\star$ denote a set of activated vertices in a known graph $G=(V,E)$ on $n$ nodes with maximal degree $d$.
We observe $C^\star$ through noisy compressed measurements of the vector $\xb = \mu \mathbf{1}_{C^\star}$, that is we may select sensing vectors $a_i \in \mathbb{R}^n$ and observe $y_i = a_i^T\xb + \epsilon_i$ where $\epsilon_i \sim \mathcal{N}(0, \sigma^2)$ independently.
We require $\sum_i ||a_i||^2 \le m$ so that the total sensing energy, or budget, is at most $m$.
We allow for \emph{adaptivity}, meaning that the procedure may use the measurements $y_1, \ldots, y_{i-1}$ to inform the choice of the subsequent vector $a_i$.
We assume that the signal strength $\mu$ is known.
Our goal is to develop procedures that successfully recover $C^\star$ in a low signal-to-noise ratio regime.



We will require the set $C^\star$, which we will henceforth call a
\emph{cluster}, to have small cut-size in the graph $G$. Formally:
\[
C^\star \in \Ccal_{\rho} = \{C : |\{(u,v) : u \in C, v \notin C\}| \le \rho\}
\]

Our algorithmic tool for identification of $C^\star$ is a \textbf{dendrogram} $\Dcal$, a hierarchical partitioning of $G$.

\begin{definition}
A \textbf{dendrogram} $\Dcal$ is a tree of blocks $\{D\}$ where each block is a connected set of vertices in $G$ and:
\begin{enumerate}
\item The root of $\Dcal$ is $V$, the set of all vertices, and the leaves of the dendrogram are all of the singletons $\{v\}$, $v \in G$.
The sets corresponding to the children of a block $D$ form a partition of the elements in $D$ while preserving graph connectivity in each cluster.
\item $\Dcal$ has degree at most $d$, the maximum degree in $G$.
\item $\Dcal$ is approximately balanced.
  Specifically the child of any block $D$ has size at most $|D|/2$.
\item The height $L$ of $\Dcal$ is at most $\log_2(n)$.
\end{enumerate}
\label{ass:dendrogram}
\end{definition}

See Figure~\ref{fig:dendrogram}.
We will see one way to construct such dendrograms in Section~\ref{sec:dendrograms}.
Note that the results of Sharpnack \emph{et. al.} imply that one can construct a suitable dendrogram for any graph~\cite{sharpnack2013}. 
By the fact that each block of $\Dcal$ is a connected set of vertices, we
immediately have the following proposition:
\begin{proposition}
A block $D$ is \textbf{impure} if $0 < |D \cap C^\star| < |D|$.
For any $C^\star$ in $\Ccal_{\rho}$ at most $\rho$ blocks are impure at any level in $\Dcal$. 
\label{prop:dend}
\end{proposition}


\subsection{Universal Guarantees}
With a dendrogram $\Dcal$, we can sense with measurements of the form $\mathbf{1}_D$ for a parent block $D$ and recursively sense on the children blocks to identify the activated vertices.
This procedure has the same flavor as the compressive binary search procedure~\cite{DavenportCastro}.
Specifically, fix a threshold $\tau$ and energy parameter $\alpha$ and when sensing on block $D$ obtain the measurement
\begin{eqnarray}
y_D = \sqrt{\alpha} \mathbf{1}_{D}^T \xb + \epsilon_D
\end{eqnarray}
If $\tau < y_D < \mu\sqrt{\alpha}|D| -\tau$ continue sensing on $D$'s children, otherwise terminate the recursion.
At a fairly weak SNR and with appropriate setting for $\tau$ and $\alpha$, we can show that this
procedure will exactly identify $C^\star$:

\begin{figure*}
\noindent \begin{minipage}{0.5\textwidth}
\begin{center}
\includegraphics[scale=0.4]{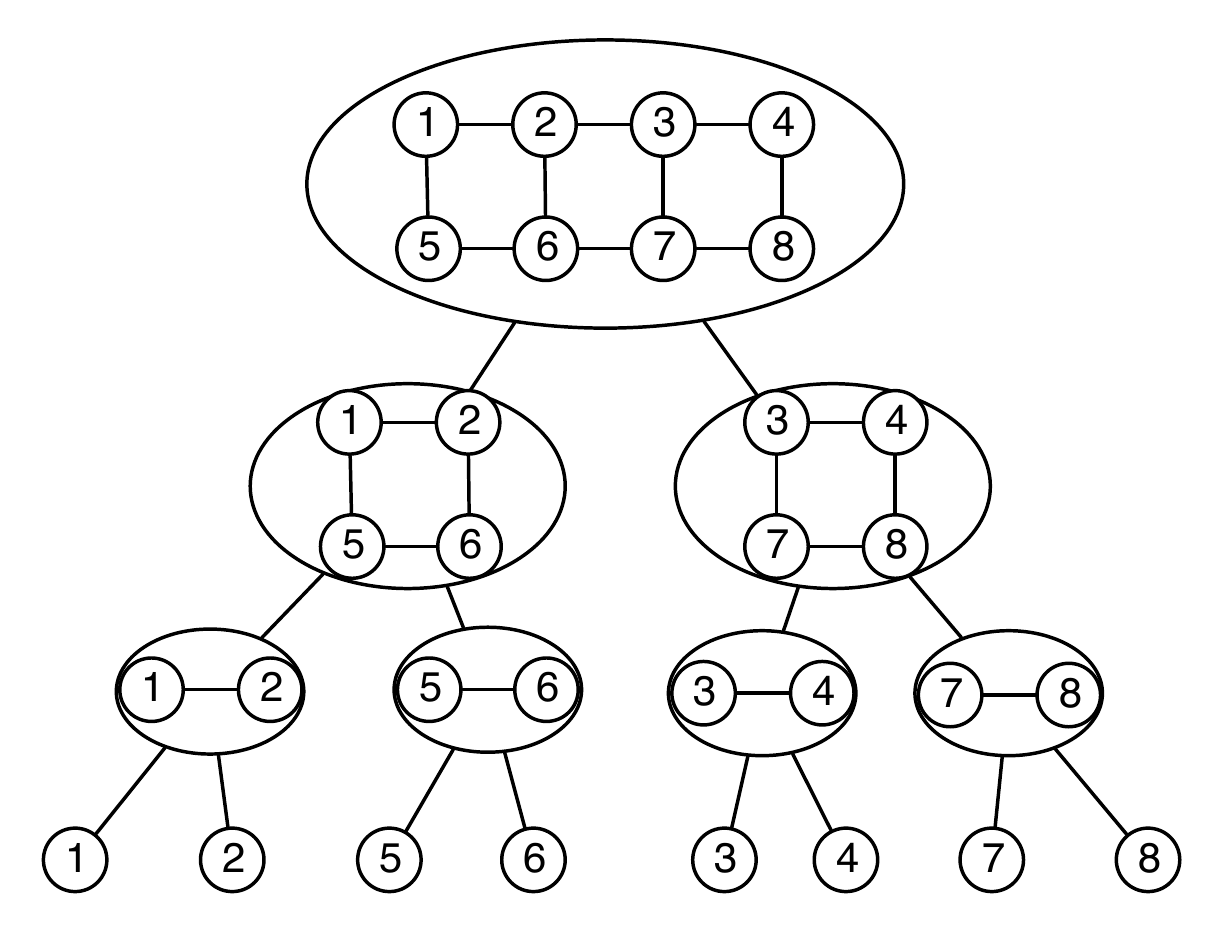}
\captionof{figure}{A dendrogram for the $2\times 4$ lattice graph.} \label{fig:dendrogram}
\end{center}
\end{minipage}%
\begin{minipage}{0.5\textwidth}
\begin{algorithm}[H]
\begin{algorithmic}
\REQUIRE Dendrogram $\Dcal$ and sensing budget $m$, failure probability $\delta$.
\STATE set $\alpha = \frac{m}{4n \log_2 \rho}$, $\tau = \sigma
\sqrt{2 \log((d \rho L+1)/\delta)}$.
\STATE (1) Let $D$ be the root of $\Dcal$.
\STATE (2) Obtain $y_D = \sqrt{\alpha} \mathbf{1}_D^T \xb + \epsilon_D$
\STATE (3) If $y_D \ge \mu \sqrt{\alpha}|D| - \tau$ add $D$ to the estimate
$\hat{C}$.
\STATE (4) If $\tau \le y_D \le \mu \sqrt{\alpha}|D| - \tau$ recurse on
(2)-(4) with $D$'s children.
\STATE Output $\hat{C}$. 
\end{algorithmic}
\caption{Exact Recovery}
\label{alg:exact}
\end{algorithm}
\end{minipage}
\end{figure*}

\begin{proposition}
Set $\tau = \sigma \sqrt{2 \log((d\rho L+1)/\delta)}$. If the SNR satisfies:
\begin{eqnarray}
\frac{\mu}{\sigma} \ge \sqrt{\frac{8}{\alpha} \log \left(\frac{d \rho
    L+1}{\delta}\right)}
\end{eqnarray}
then with probability $\ge 1-\delta$, Algorithm~\ref{alg:exact} recovers $C^\star$ and using a sensing budget of at most $3 n \alpha \log_2 (d \rho)$.
\label{prop:exact}
\end{proposition}

We must set $\alpha \le \frac{m}{3n \log_2(d\rho)}$ so we do not exceed our budget of $m$.
With the this setting, the SNR requirement is:
\[
\frac{\mu}{\sigma} \ge \sqrt{\frac{24 n}{m} \log_2 (d \rho) \log \left(\frac{d \rho
    L+1}{\delta}\right)}
\]

Algorithm~\ref{alg:exact} performs similarly to the adaptive procedures for unstructured support recovery.
For constant $\rho$, the SNR requirement is $\omega(\sqrt{\frac{n}{m} \log \log_2 n})$ which is on the same order as the compressive binary search procedure~\cite{DavenportCastro} for recovering $1$-sparse signals.
For $k$-sparse signals, the best results require SNR of $\sqrt{\frac{n \log k}{m}}$ which can be much worse than our guarantee when $k \ge \log n$ and $\rho$ is small~\cite{MalloyNowak,HauptNowak}.

Thus, the procedure does enjoy small benefit from exploiting structure, but the generality of our set up precludes more substantial performance gains.
Indeed, we are able to show that one cannot do much better than Algorithm~\ref{alg:exact}.
This information theoretic lower bound is a simple consequence of the results from Arias-Castro \emph{et. al.}~\cite{arias2011fundamental}.

\begin{theorem}
Fix any graph $G$ and suppose $\rho \ge d$. If:
\[
\frac{\mu}{\sigma} = o\left(\sqrt{\frac{n}{m}}\right)
\]
then $\inf_{\hat{C}} \sup_{C^\star \in \Ccal_\rho} \mathbb{P}[\hat{C} \ne C^\star] \rightarrow \frac{1}{2}$. Therefore no procedure can reliably estimate $C^\star \in \Ccal_\rho$.
\label{thm:lower}
\end{theorem}

The lower bound demonstrates one of the fundamental challenges in exploiting structure in the cluster recovery problem: since $\Ccal_\rho$ is not parameterized by cluster size, in the worst case, one should not hope for performance improvements that depend on cluster size or sparsity. 
More concretely, if $\rho \ge d$, the set $\Ccal_{\rho}$ contains all singleton vertices, reducing to a completely unstructured setting.
Here, the results of Davenport and Arias-Castro imply that to exactly recover a cluster of size one, it is necessary to have SNR of $\sqrt{\frac{n}{m}}$~\cite{DavenportCastro}.
Moreover, nothing in our setup prevents $G$ from being a complete graph on $n$ vertices, which also reduces to the unstructured setting.

The inherent difficulty of this problem is not only information-theoretic, but also computational.
The typical way to exploit structure is to scan across the possible signal patterns, using the fact that the search space is highly restricted.
Unfortunately, Karger proved that the number of cuts of size $\rho$ is $\Theta(n^\rho)$ \cite{Karger96}, meaning that $\Ccal_\rho$ is not very restrictive.
Even if we could efficiently scan all patterns in $\Ccal_\rho$, distinguishing between two clusters with high overlap would still require high SNR.
As a concrete example, Balakrishnan \emph{et. al.} showed that localizing a contiguous chain of activations in a line graph is impossible when $\frac{\mu}{\sigma} = o(\max \{\frac{1}{k}\sqrt{\frac{n-k}{m}}, \sqrt{\frac{1}{m}}\})$~\cite{balakrishnan2012recovering}.
The second term arises from the overlap between the contiguous blocks and is independent of both $k$ and $n$, demonstrating the challenge in distinguishing these overlapping clusters.


\subsection{Cluster-Specific Guarantees}
The main performance bottleneck for Algorithm~\ref{alg:exact} comes from testing whether a block of size 1 is active or not.
If there are no such singleton blocks, meaning that the cluster $C^\star$ is grouped into large blocks in $\Dcal$, we might expect that Algorithm~\ref{alg:exact} or a variant can succeed at lower SNR.
We formalize this idea here, analyzing an algorithm whose performance depends on how $C^\star$ is partitioned across the dendrogram $\Dcal$. 

We quantify this dependence with the notion of \emph{maximal} blocks $D \in \Dcal$ which are the largest blocks that are completely active.
Formally $D$ is \emph{maximal} if $D \cap C^\star = D$ and $D$'s parent is impure, and we denote this set of maximal blocks $\Mcal$.
If the maximal blocks are all large, then we can hope to obtain performance improvements.


The algorithm consists of two phases.
The first phase (the adaptive phase) is similar to Algorithm~\ref{alg:exact}.
With a threshold $z$, and energy parameter $\alpha$, we sense on a block $D$ with
\[
y_D = \sqrt{\alpha}\mathbf{1}_{D}^T\xb+\epsilon_D
\]
If $y_D > z$ we sense on $D$'s children and we construct a \textbf{pruned dendrogram} $\Kcal$ of all blocks $D$, for which $y_D > z$.
The pruned dendrogram is much smaller than $\Dcal$ but it retains a large fraction of $C^\star$.

Since we have significantly reduced the dimensionality of the problem we can now use a passive localization procedure to identify $C^\star$ at a low SNR.
In the passive phase, we construct an orthonormal basis $U$ for the subspace:
\[
\{\mathbf{1}_D : D \in \Kcal\}
\]
With another energy parameter $\beta$, we observe $y_{i} = \sqrt{\beta}u_i^T\xb + \epsilon_i$ for each basis vector $u_i$ and form the vector $\yb = \sqrt{\beta} U^T\xb + \mathbf{\epsilon}$ by stacking these observations.
We then construct the vector $\hat{\xb} = U\yb/\sqrt{\beta}$.
With the vector $\hat{\xb}$ we solve the following optimization problem to identify the cluster ($[n] = \{1,\ldots, n\}$):
\[
\hat{C} = \textrm{argmax}_{C \subseteq [n]} \frac{\mathbf{1}_C^T\hat{\xb}}{||\hat{\xb}||\sqrt{|C|}}
\]
which can be solved by a simple greedy algorithm.
A detailed description is in Algorithm~\ref{alg:approx}.
For a more concise presentation, in the following results, we omit the dependence on the maximum degree of the graph, $d$.
This localization guarantee is stated in terms of the distance $d(\hat{C}, C^\star) \triangleq 1 - \frac{|\hat{C} \cap C^\star|}{\sqrt{|\hat{C}||C^\star|}}$.

\begin{algorithm}[t]
\begin{algorithmic}
\REQUIRE Dendrogram $\Dcal$, sensing budget parameters $\alpha, \beta$.
\STATE Set $\alpha, z$ as in Theorem~\ref{thm:approx}. Initialize $\Kcal = \emptyset$.
\STATE (1) Let $D$ be the root of $\Dcal$.
\STATE (2) Obtain $y_D = \sqrt{\alpha} \mathbf{1}_D^T \xb + \epsilon_D$.
\STATE (3) If $y_D \ge z$ add $D$ to $\Kcal$ and recurse on (1)-(3) with $D$s children.
\STATE Construct $U$ an orthonormal basis for $\textrm{span}\{\mathbf{1}_D\}_{D
  \in \Kcal}$.
\STATE Sense $\yb = \sqrt{\beta}U^T \xb +
\epsilon$ and form $\hat{\xb} = U \yb/\sqrt{\beta}$.
\STATE Output $\hat{C} = \textrm{argmax}_{C \subseteq [n]}
\frac{\mathbf{1}_C^T\hat{\xb}}{||\hat{\xb}|| \sqrt{|C|}}$. 
\end{algorithmic}
\caption{Approximate Recovery}
\label{alg:approx}
\end{algorithm}

\begin{theorem}
Set $z$ so that $\mathbb{P}[\mathcal{N}(0,1) > \sigma z] \le \frac{\sqrt{5}-1}{d}$ and \footnote{We provide exact definitions of $\alpha$ and $\beta$ in the \ifthenelse{\equal{\version}{arxiv}}{appendix}{supplementary material~\cite{krishnamurthy2013arxiv}}.}
\[
\alpha = \frac{m}{n \log_2((\rho+k)\log n)}, \beta = \frac{m}{(\rho+k)\textnormal{polylog}(n,\rho)}
\]
where $k = |C^\star|$. If
\[
\frac{\mu}{\sigma} = \omega \left( \frac{(\rho + k) \textnormal{polylog}(n,\rho)}{\sqrt{m k}} + \sqrt{\frac{n \log_2((\rho+k) \log n)}{m|M_{\min}|^2}}\right)
\]
where $M_{\min} = \textrm{argmin}_{M \in \Mcal} M$, then $d(\hat{C}, C^\star) \rightarrow 0$ and the budget is $O(m)$. 
\label{thm:approx}
\end{theorem}

The SNR requirement in the theorem decomposes into two terms, corresponding to the two phases of the algorithm, and our choice of $\alpha$ and $\beta$ distribute the sensing budget evenly over the terms, allocating $O(m)$ energy to each.
Note however, that the first term, corresponding to the passive phase, has a logarithmic dependence on $n$ while the second term, corresponding to the adaptive phase, has a polynomial dependence, so in practice one should allocate more energy to the adaptive phase.
With our allocation, the second term usually dominates, particularly for small $\rho$ and $k$, which is a regime of interest.
Then the required SNR is:

\[
\frac{\mu}{\sigma} = \omega\left(\frac{1}{|M_{\min}|}\sqrt{\frac{n}{m}
  \log_2((\rho + k) \log n)}\right)
\]

\begin{table}
\centering
\begin{tabular}{|c |c|}
\hline
Setting & $\frac{\mu}{\sigma}$\\
\hline\hline
One maximal block & $\omega \left(\frac{1}{k}\sqrt{\frac{n}{m} \log(k \log n) }\right)$\\
Uniform sizes & $\omega\left( \frac{\rho}{k}\sqrt{\frac{n}{m} \log(k \log n)}\right)$\\
Worst Case & $\omega\left( \sqrt{\frac{n}{m} \log(k \log n)}\right)$\\
\hline
\end{tabular}
\caption{Instantiations of Theorem~\ref{thm:approx}}
\label{tab:realizations}
\ifthenelse{\equal{\version}{arxiv}}{}{\vspace{-0.7cm}}
\end{table}

To more concretely interpret the result, we present sufficient SNR scalings for three scenarios in Table~\ref{tab:realizations}. We think of $\rho \ll |C^\star|$. 
The most favorable realization is when there is only one maximal block of size $k$. 
Here, there is a significant gain in SNR over unstructured recovery or even Algorithm~\ref{alg:exact}.

Another interesting case is when the maximal blocks are all at the same level in the dendrogram.
In this case, there can be at most $\rho d$ maximal blocks since each of the parents is impure and there can only be $\rho$ impure blocks per level.
If the maximal blocks are approximately the same size, then $|M_{\min}| \approx k/\rho$, and we arrive at the requirement in the second row of Table~\ref{tab:realizations}.
Again we see performance gains from structure, although there is some degradation.

Unfortunately, since the bound depends on $M_{\min}$, we do not always realize such gains.
When $M_{\min}$ is a singleton block (one node), our bound deteriorates to the third row of Table~\ref{tab:realizations}.
We remark that modulo $\log \log$ factors, this matches the SNR scaling for the unstructured (sparse) setting.
It also nearly matches the lower bound in Theorem~\ref{thm:lower}.

Theorem~\ref{thm:approx} shows that the size of $|M_{\min}|$ is the bottleneck to recovering $C^\star$.
If we are willing to tolerate missing the small blocks we can sense at lower SNR.

\begin{corollary}
Let $\tilde{C} = \bigcup_{M \in \Mcal, |M| \ge t} M$ and $k = |C^\star|$.
If:
\[
\frac{\mu}{\sigma} = \omega\left( \frac{(\rho+k)\textnormal{polylog}(n,\rho)}{\sqrt{m k}} + \frac{1}{t}\sqrt{\frac{n}{m}\textnormal{polylog}(n, \rho, j, t)}\right)
\]
then with probability $1-o(1)$, $d(\hat{C}, \tilde{C}) \rightarrow 0$ and $n \rightarrow \infty$.
\label{cor:approximate}
\vspace{-0.25cm}
\end{corollary}

In particular, we can recover all maximal blocks of size $t$ with SNR on the order of $\tilde{O}(\frac{1}{t}\sqrt{\frac{n}{m}})$, which clearly shows the gain in exploiting structure in this problem.

\subsection{Constructing Dendrograms}
\label{sec:dendrograms}
A general algorithm for constructing a dendrogram parallels the construction of spanning tree wavelets in Sharpnack \emph{et. al.}~\cite{sharpnack2013}.
Given a spanning tree $\Tcal$ for $G$, the root of the dendrogram is $V$, and the children are the subtrees around a balancing vertex $v \in \Tcal$.
The dendrogram is built recursively by identifying balancing vertices and using the subtrees as children.
See Algorithm~\ref{alg:dendrogram} for details.
It is not hard to verify that this algorithm produces a dendrogram according to Definition~\ref{ass:dendrogram}.

\begin{algorithm}[t]
\begin{algorithmic}
\REQUIRE $\Tcal$ a subtree of $G$ and initialize $v \in \Tcal$ arbitrarily .
\STATE (1) Let $T'$ be the component of $\Tcal \backslash \{v\}$ of largest size.
\STATE (2) Let $w$ be the unique neighbor of $v$ in $T'$.
\STATE (3) Let $T''$ be the component of $\Tcal \backslash \{w\}$ of largest size.
\STATE (4) Stop and return $v$ if $|T''| \ge |T'|$.
\STATE (5) $v \leftarrow w$. Repeat at (1).
\end{algorithmic}
\caption{FindBalance}
\label{alg:findbalance}
\end{algorithm}

\begin{algorithm}[t]
\begin{algorithmic}
\REQUIRE $\Tcal$ is a spanning tree of $G$.
\STATE Initialize $\Dcal = \{\{v : v \in \Tcal\}\}$.
\STATE Let $v$ be the output of \textrm{FindBalance} applied to $\Tcal$. 
\STATE Let $\Tcal_1, \ldots, \Tcal_{d_v}$ be the connected component of $\Tcal
\setminus v$ and add $v$ to the smallest component. 
\STATE Add $\{v : v \in \Tcal_i\}$ for each $i$ as children of $\Tcal$ to
$\Dcal$.
\STATE Recurse at (2) for each $\Tcal_i$ as long as $|\Tcal_i| \ge 2$. 
\end{algorithmic}
\caption{BuildDendrogram}
\label{alg:dendrogram}
\end{algorithm}

%% file: experiments.tex
\section{Experiments}
We conducted two simulation studies to verify our theoretical results and examine the performance of our algorithms empirically.
First, we empirically verify the SNR scaling in Proposition~\ref{prop:exact}.
In the second experiment, we compare both of our algorithms with the algorithm of Haupt and Nowak~\cite{HauptNowak}, which is an unstructured adaptive compressed sensing procedure with state-of-the-art performance.

\begin{figure}
\centering
\ifthenelse{\equal{\version}{arxiv}}{
\includegraphics[scale=0.3]{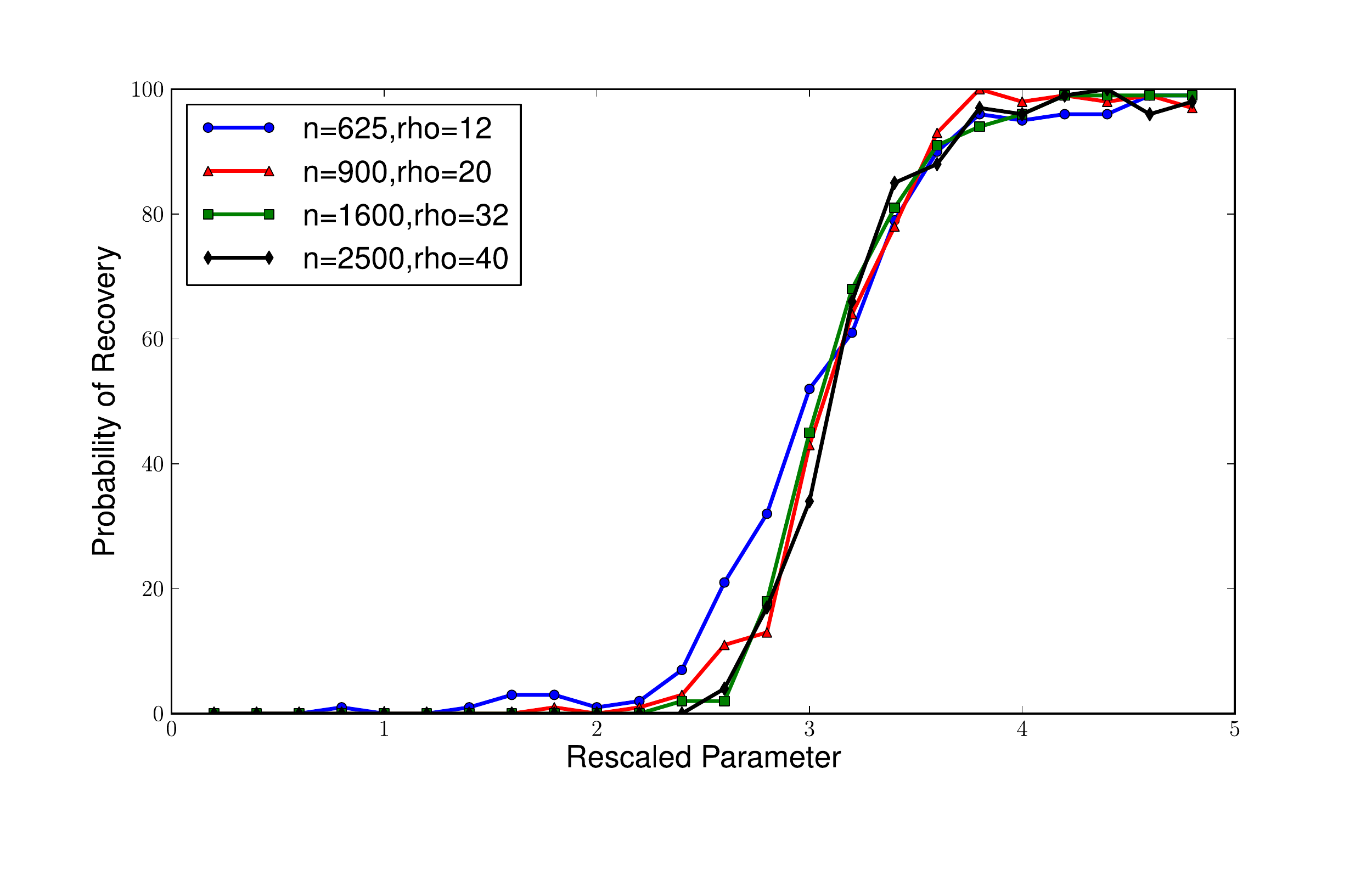}\vspace{-0.7cm}
}{
\includegraphics[scale=0.3]{./exact_adaptive_threshold_torus.pdf}\vspace{-0.7cm}
}
\caption{Probability of success for Algorithm~\ref{alg:exact} as a function of the
  $\theta = \frac{\mu}{\sigma} \sqrt{\frac{m}{n \log_2 \rho
      \log (\rho \log(n))}}$ for the torus.}
\label{fig:thresholds}
\ifthenelse{\equal{\version}{arxiv}}{}{\vspace{-0.7cm}}
\end{figure}

\ifthenelse{\equal{\version}{arxiv}}{
\begin{figure}[t]
\centering
\includegraphics[scale=0.30]{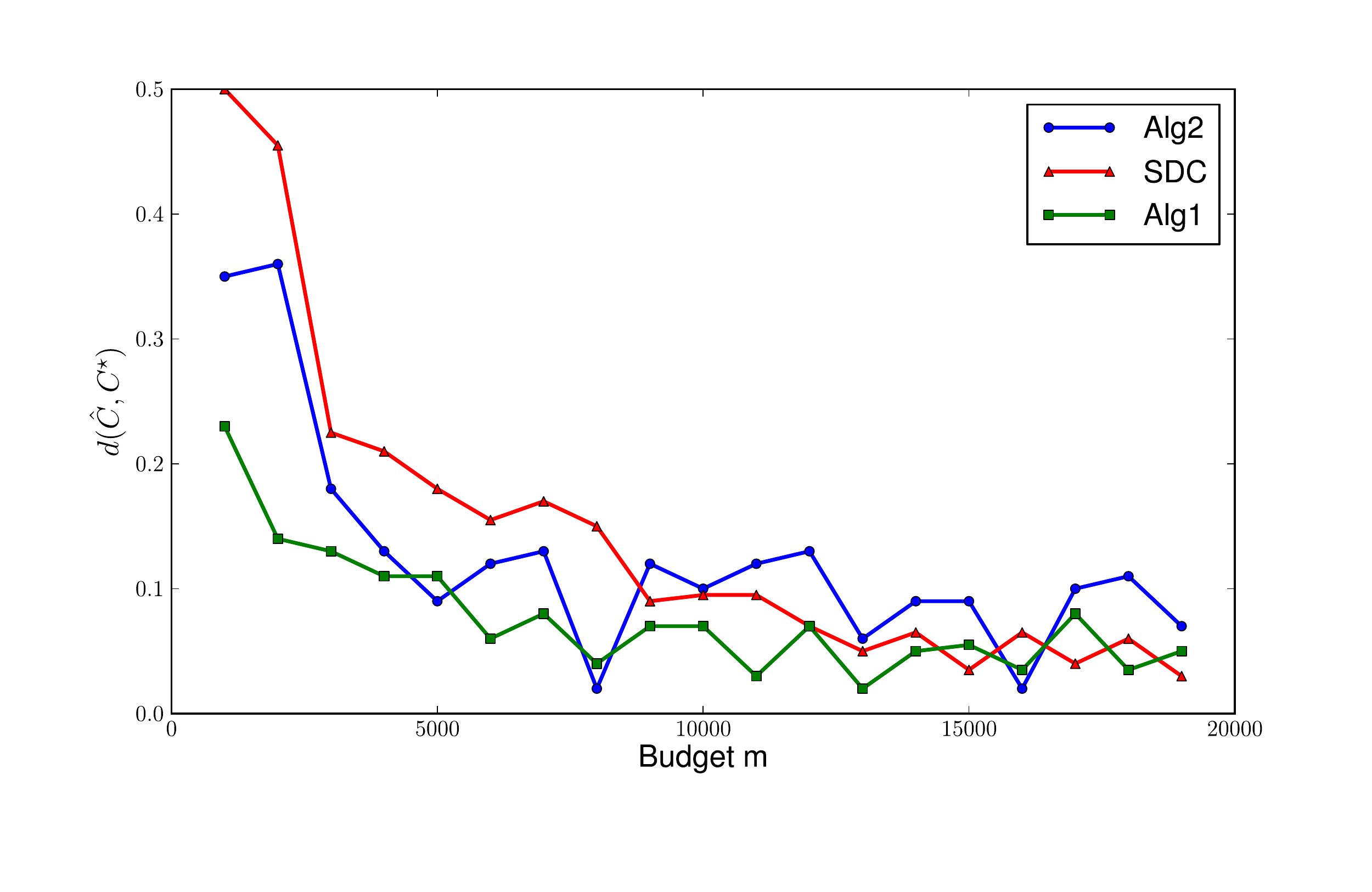}\hspace{-0.5cm}
\includegraphics[scale=0.30]{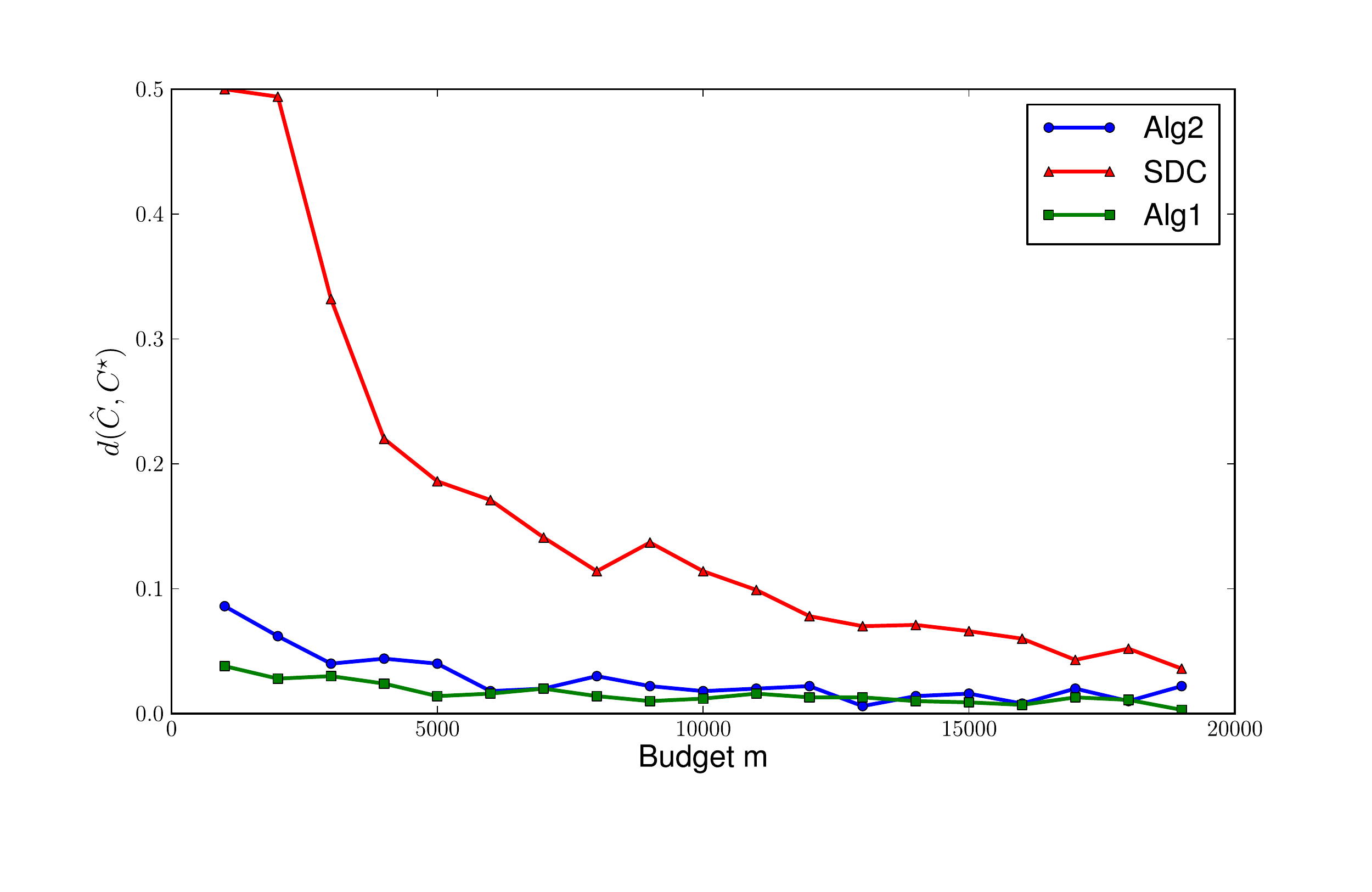}
\caption{Error as a function of $m$ for $n=512$ and $k=10,50$ (top, bottom)
  demonstrating the gains from exploiting structure. Here $G$ is a line graph
  and $\rho=2$, resulting in one connected cluster.}
\label{fig:errors}
\end{figure}
}{
\begin{figure}
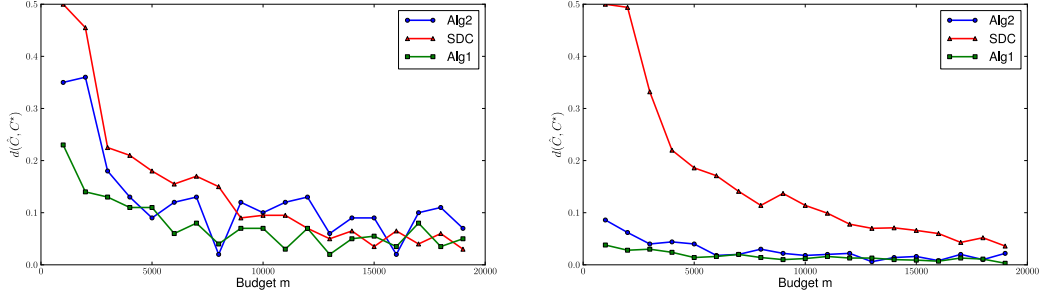

\centering
\includegraphics[scale=0.30]{./error_scaling_512_10.pdf}\\\vspace{-0.7cm}
\includegraphics[scale=0.30]{./error_scaling_512_50.pdf}\vspace{-0.7cm}
\caption{Error as a function of $m$ for $n=512$ and $k=10,50$ (top, bottom)
  demonstrating the gains from exploiting structure. Here $G$ is a line graph
  and $\rho=2$, resulting in one connected cluster.}
\label{fig:errors}
\vspace{-0.7cm}
\end{figure}
}
In Figure~\ref{fig:thresholds} we plot the probability of successful recovery of $C^\star$ as a function of a rescaled parameter.
This parameter $\theta(n,m,\rho,\frac{\mu}{\sigma}) = \frac{\mu}{\sigma} \sqrt{\frac{m}{n \log_2 \rho \log(\rho \log n)}}$ was chosen so that the condition on the SNR in Proposition~\ref{prop:exact} is equivalent to $\theta = c$ for some constant $c$.
Proposition~\ref{prop:exact} then implies that with this rescaling, the curves should all line up, which is the phenomenon we observe in Figure~\ref{fig:thresholds}.
Here $G$ is the two dimensional torus and $\Dcal$ was constructed using Algorithm~\ref{alg:dendrogram}.

In Figure~\ref{fig:errors} we plot the error, measured by $d(\hat{C}, C^\star)$, as a function of $m$ for three algorithms.
We use both Algorithms~\ref{alg:exact} and~\ref{alg:approx} as well as the sequentially designed compressed sensing algorithm (SDC)~\cite{HauptNowak},which has near-optimal performance for unstructured sparse recovery. 
Here $G$ is the line graph, $\Dcal$ is the balanced binary dendrogram, and $\rho=2$ so each signal is a contiguous block.

In the first figure, $k=10$ and since the maximal clusters are necessarily small, there should be little benefit from structure.
Indeed, we see that all three algorithms perform similarly.
This demonstrates that in the absence of structure, our procedures perform comparably to existing approaches for unstructured recovery.
When $k=50$ (the second figure), we see that both Algorithms~\ref{alg:exact} and~\ref{alg:approx} outperform SDC, particularly at low SNRs.
Here, as predicted by our theory, Algorithm~\ref{alg:approx} can identify a large part of the cluster at very low SNR by exploiting the cluster structure.
In fact Algorithm~\ref{alg:exact} empirically performs well in this regime although we do not have theory to justify this.

%% file: conclusion.tex
\section{Conclusion}
We explore the role of structure and adaptivity in the support recovery problem, specifically in localizing a cluster of activations in a network.
We show that when the cluster has small cut size, exploiting this structure can result in performance improvements in terms of signal-to-noise ratios sufficient for cluster recovery.
If the true cluster $C^\star$ coincides with a dendrogram over the graph, then weaker SNRs can be tolerated.
These results do not contradict the necessary conditions for this problem, which shows that one cannot do much better than the unstructured setting for exact recovery.

While our work contributes to understanding the role of structure in compressive sensing, our knowledge is still fairly limited.
We now know of some specific instances where structured signals can be localized at a very weak SNR, but we do not have a full characterization of this effect.
Our goal was to give such a precise characterization, but the generality of our set-up resulted in an information-theoretic barrier to demonstrating significant performance gains.
An interesting direction for future research is to precisely quantify settings that are not too general nor very specific when structure can lead to improved sensing performance and to develop algorithms that enjoy these gains.

%% file: appendix.tex
\section{Proof of Theorem~\ref{thm:lower}}
The proof is a simple extension of Theorem 2 from Davenport and
Arias-Castro~\cite{DavenportCastro}. In particular, if $\rho > d$ then
$\Ccal_{\rho}$ contains all one-sparse signals. Restricting to just these
signals, the results from~\cite{DavenportCastro} imply that we cannot even
detect if the activation is in the first or second half of the vertices unless
$\frac{\mu}{\sigma} \ge \sqrt{\frac{n}{m}}$. This results in the lower bound.

If we are also interested in introducing the cluster size parameter $k$ we are
can prove a similar lower bound by reduction to one-sparse testing. If $\rho > k
d$ then all ${n \choose k}$ support patterns are in $\Ccal_{\rho}$ so we are
again in the unstructured setting. Here, the results
from~\cite{arias2011fundamental} give the lower bound. 

If $\rho < k d$ then we are in a structured setting in that not all ${n \choose
  k}$ support patterns are possible. However, if we look at the cycle graph,
each contiguous block contributes $2$ to the cut size, so if $\rho \ge 4$ we are
allowed at least two contiguous blocks. If $k-1$ of the activations lie in one
contiguous block, then the last activation can be distributed in any of the
$n-k+1$ remaining vertices. Even if the localization procedure was provided with
knowledge of the location of the $k-1$ activations, an SNR of
$\frac{\mu}{\sigma} \ge \sqrt{\frac{n-k}{m}}$ would be necessary for identifying
the last activation.

\section{Proof of Proposition~\ref{prop:exact}}
Recall that for any block $D$ that we sense, we obtain $y_D =
\sqrt{\alpha}\mathbf{1}_D^T \mu \mathbf{1}_{C^\star} + \epsilon_d$. Consider a
single block $D$, Gaussian tail bounds reveal the following facts:
\begin{enumerate}
\item If $D \cap C^\star = \emptyset$, then with probability $\ge 1-\delta$, $y_D \le
  \sigma \sqrt{2 \log(1/\delta)}$.
\item If $D \cap C^\star = D$, then with probability $\ge 1-\delta$, $y_d \ge
  \mu \sqrt{\alpha} |D| - \sigma \sqrt{2 \log(1/\delta)}$.
\item Otherwise, with probability $1-\delta$: $\mu \sqrt{\alpha} - \sigma
  \sqrt{2 \log(1/\delta)} \le y_D \le \mu \sqrt{\alpha} (|D|-1) + \sigma \sqrt{2
    \log(1/\delta)}$.
\end{enumerate}
The above facts reveal that if:
\[
\frac{\mu}{\sigma} \ge 2\sqrt{\frac{2 \log(1/\delta)}{\alpha}}
\]
then we will correctly identify if $D$ is empty, full or impure. Assuming we
perform this test correctly, we only refine $D$ if it is impure, and
Proposition~\ref{prop:dend} reveals that at most $\rho$ clusters can be impure
per level. For each of these $\rho$ clusters that we refine, we search on at
most $d_{\max}$ clusters at the subsequent level. The total budget that we use
is (recall that $L$ is the height of $\Dcal$):
\begin{eqnarray*}
\sum_{l=0}^L \alpha \min\{n, \rho d \frac{n}{2^l}\} &=& \alpha \left(\sum_{l=0}^{\log_2(\rho d) -1} n + \sum_{l=\log_2(\rho d)}^L \frac{\rho d n}{2^l}\right)\\
& = & \alpha \left(n \log_2(\rho d) + \sum_{l=0}^{L - \log_2(\rho d)}  \frac{\rho}{2^{\log_2(\rho d)}} \frac{n}{2^l}\right)\\ 
&\le& \alpha (n \log_2(\rho d) + 2n) \le 3\alpha n \log_2(\rho d)
\end{eqnarray*}
Setting $\alpha$ as in the Proposition makes this quantity smaller than $m$. 
Finally, we take a union bound over the $\rho d L+1$ blocks that we sense on ($\rho d L$ per level not counting the root and one more for the root) and plug in our bound on $\alpha$ to arrive at the final rate of:
\[
\frac{\mu}{\sigma} \ge \sqrt{\frac{24 n \log_2(\rho d) \log((\rho d L+1)/\delta)}{m}}
\]
The threshold $\tau$ is specified to ensure that that failure probability for all of the tests is at most $\delta$.
In the algorithm, thresholding at $\tau$ and $\mu \sqrt{\alpha} - \tau$ favors identifying clusters as impure over identifying them as pure.
In practice these thresholds may improve performance because it is worse to incorrectly mark an impure block as pure (leading to incorrect recovery of the cluster) than it is to incorrectly mark a pure block as impure (leading to an increase in measurement budget).
Using these thresholds has no ramifications on our theory. 

\section{Proof of Theorem~\ref{thm:approx}}

To prove Theorem~\ref{thm:approx} we must analyze each phase of the procedure.
We first turn to the adaptive phase. 
By setting the threshold $z$ correctly, we retain a large fraction of $C^\star$ while removing a large number of inactive nodes.
We call the set of blocks retained by the adaptive phase $\Kcal$ and we measure the fraction of $C^\star$ lost by the projection onto the basis $U$ for the subspace spanned by the blocks in $\Kcal$.
In the passive phase, we use the fact that $|\Kcal|$ is small to bound the MSE of the reconstruction $\mathbb{E}||\hat{\xb}-\xb||^2$.
We then translate this MSE guarantee into an error guarantee for $\hat{C}$.

Throughout, let $r$ denote the total number of impure blocks in the dendrogram and let $L$ denote the height of $\Dcal$. 
Note that $r \le \rho L$ and $L \le \lceil \log_2 n \rceil$.
Recall that $k = |C^\star|$.

With all the results in the following sections we will be able to bound
$d(\hat{C}, C^\star)$ with probability $\ge 1 - 3\delta$ as:
\begin{eqnarray}
d(\hat{C}, C^\star) &\le & \frac{4}{\mu^2k} ||\hat{\xb} -
\mu\mathbf{1}_{C^\star}||^2 \label{eq:recovery_use}\\ 
& \le &
\frac{4 c\sigma^2|\Kcal|}{\mu^2 k \beta} + \frac{4L}{\delta}\exp\left\{-1/8
\alpha |M_{\min}|^2 \mu^2/\sigma^2\right\} \label{eq:mse_use}\\ 
& \le & \frac{8 c \sigma^2L^2 (3rd \log(rdL/\delta) + k)^2}{\mu^2km} + \\ 
& & +  \frac{4 L}{\delta} \exp\left\{-\frac{1}{48} \frac{m |M_{\min}|^2\mu^2/\sigma^2}{n \log_2(4rd^2\log(rdL/\delta) + k)}\right\}
\end{eqnarray}
Here Equation~\ref{eq:recovery_use} follows from our analysis of the optimization phase (Lemma~\ref{lem:recovery}), and Equation~\ref{eq:mse_use} follows from the bounds in Section~\ref{sec:passive}.
The last step follows by plugging in bounds on $\alpha$ and $\beta$ if we want to allocate $m/2$ energy to each phase.
Specifically the bound on $\alpha$ comes from Lemma~\ref{lem:adapt_energy} while the bound for $\beta$ comes from Lemma~\ref{lem:adapt_empty}.
Lemma~\ref{lem:adapt_energy} shows that the energy used in the first phase of the algorithm is $\le \alpha (3n \log_2(4rd^2 \log(rdL/\delta) + k))$ and allocating $m/2$ energy to this phase gives the bound on $\alpha$.
On the other hand Lemma~\ref{lem:adapt_empty} shows that the after the adaptive phase, the pruned dendrogram contains at most $L (3rd \log(rdL/\delta) + k)$ blocks, which is precisely the dimensionality of the subspace we sense over in the passive phase of the procedure. 
To summarize, setting:
\begin{eqnarray}
\alpha &\le& \frac{m}{6n \log_2(4rd^2 \log(rdL/\delta) + k)}\\
\beta &\le& \frac{m}{6 rdL  \log(rdL/\delta) + 2Lk}
\end{eqnarray}
ensures that the sensing budget is no more than $m$. 

We obtain the final result by plugging in the bounds on $L \le \lceil \log_2 n \rceil$ and $r \le \rho \lceil \log_2 n \rceil$. With these bounds, the first term is $o(1)$ as long as:
\[
\frac{\mu}{\sigma} = \omega\left(\frac{\rho d \log_2^2n \log(\rho d
  \log_2^2n/\delta) + k \log_2 n}{\sqrt{m k}}\right)
\]
The second term is $o(1)$ when:
\[
\frac{\mu}{\sigma} = \omega\left(\frac{1}{|M_{\min}|}\sqrt{\frac{n}{m} \log_2(\rho d^2  \log_2 n
  \log (\rho d \log_2^2n/\delta) + k)\log(\log_2 n/\delta) }\right)
\]
Note that we can only apply Lemma~\ref{lem:recovery} if the right hand side of Equation~\ref{eq:recovery_use} is $\le 1$.
However if $\mu/\sigma$ meets the above two requirements, then that quantity is going to zero, so for $n$ large enough this will certainly be the case.

\subsection{The Adaptive Phase}

Our analysis will focus on recovering \emph{maximal} blocks $D \in \Dcal$, which
are the largest blocks that contain only activated vertices. Formally, $D$ is
maximal if $D \cap C^\star = D$ and if $D$'s parent contains some unactivated
vertices. We are also interested in identifying \emph{impure} blocks, (blocks that
partially overlap with $C^\star$). Suppose there are $r$ such impure
clusters.

The first lemma helps us bound the number empty nodes that we retain:
\begin{lemma}
Threshold at $\sigma z$ where $\mathbb{P}[\mathcal{N}(0,1) > z] \le q$ and:
\[
q = \frac{\sqrt{5}-1}{2 d_{\max}}
\]
Then with probability $\ge 1-\delta$ the pruned dendrogram contains at most $3rd
\log(rdL/\delta) + |C^\star|$ blocks per level for a total of at most $L(3rd
\log(rdL/\delta) + |C^\star|)$.
\label{lem:adapt_empty}
\end{lemma}
\begin{proof}
For the first claim, we analyze the adaptive procedure on an empty dendrogram,
showing that we retain no more than $3 \log(L/\delta)$ per level. The proof is
by induction on the level $l$. Let the inductive hypothesis be that $t_l \le 3
\log (L/\delta)$ where $t_l$ is the number of nodes retained at the $l$th
level. Then by the Chernoff bound,
\[
\mathbb{P}[t_l - \mathbb{E}t_l \ge \epsilon] \le \exp\{ \frac{-\epsilon^2}{3 \mathbb{E}t_l}\}
\]
$\mathbb{E}t_l$ can be bounded by $dqt_{l-1}$ since each of the blocks that we
retain at the $l-1$st level can have at most $d$ children and since we retain
each block with probability $q$ in expectation. With a union bound across all
$L$ levels, we have that with probability $\ge 1-\delta$:
\[
t_l \le dq t_{l-1} + \sqrt{3dqt_{l-1} \log(L/\delta)}
\]
Applying the inductive hypothesis and the definition of $q$:
\[
t_l \le 3 \log(L/\delta) (dq + \sqrt{dq}) \le 3 \log(L/\delta)
\]
Thus for each empty dendrogram, we retain at most $3L \log(L/\delta)$.

Each of the $r$ impure clusters can spawn off at most $d$ empty subtrees in the
dendrogram. Taking a union bound over each of these $rd$ empty subtrees shows
that at most $3rdL \log(rdL/\delta)$ empty blocks are retained. There are at
most $|C^\star| L$ active blocks, which gives us a bound on the size of $\Kcal$.
\end{proof}

Next we compute the probability that we fail to retain a maximal cluster:
\begin{lemma}
For any maximal cluster $M$, the probability that $M \notin \Kcal$ is bounded
by:
\[
\mathbb{P}[M \notin \Kcal] \le L \exp\{-1/2 (\sqrt{\alpha}|M| \mu/\sigma - z)^2\}
\]
as long as $\sqrt{\alpha} |M| \mu > \sigma z$.
\label{lem:adaptive_keep}
\end{lemma}
\begin{proof}
We fail to retain a maximal cluster $M$ if we throw away any of its ancestors in
the dendrogram. All the ancestors of $M$ have at least $M$ activations so
$\mathbb{E}y_D \ge \mu \sqrt{\alpha} |M|$ for each of $M$'s ancestors. All $y_D$
have the same variance $\sigma^2$. By a union bound and Gaussian tail inequality
the failure probability is at most:
\[
\mathbb{P}[M \notin \Kcal] \le L \mathbb{P}[y_M < \sigma z] \le L \exp\{-1/2
(\sqrt{\alpha}|M| \mu/\sigma - z)^2\}
\]
\end{proof}

To complete the adaptive phase, we must set $\alpha$ so that we use at most half
of the budget.
\begin{lemma}
The energy used in the adaptive phase is:
\[
\alpha (3n \log_2(4rd^2 \log(rdL/\delta) + |C^\star|))
\]
\label{lem:adapt_energy}
\end{lemma}
\begin{proof}
At level $l$ we retain at most $3rd\log(rdL/\delta)$ empty blocks, so we sense
on at most $3rd^2 \log(rdL/\delta) + (d-1)\rho$ empty blocks (the at most $\rho$
impure blocks could spawn off up to $d-1$ empty ones). We also sense on at most
$\rho$ impure blocks and also sense every completely active block. In total we
sense on no more:
\[
3rd^2 \log(rdL/\delta) + d\rho + |C^\star| \le 4rd^2 \log (rdL/\delta) + |C^\star|
\]
blocks ($\rho \le r$) at the $l+1$st level. Since each block at the $l$th level
has size at most $n/2^l$ we can bound the total energy as:
\begin{eqnarray*}
&& \alpha \sum_{l=0}^L \min\{n, (4rd^2 \log(rdL/\delta) + |C^\star|)
  \frac{n}{2^l}\}\\ & \le & \alpha \left(n \log_2(4rd^2 \log(rdL/\delta) + |C^\star|)
  + \sum_{l=0}^{\infty} \frac{n}{2^l}\right)\\ & \le & \alpha \left(3n
  \log_2(4rd^2 \log(rdL/\delta) + |C^\star|)\right)
\end{eqnarray*}
Here to arrive at the second line, we noticed that at the top levels, sensing on
all of the blocks is a sharper bound than the one we computed which produces the
first term. The second term comes from the fact that since we sense a constant
number of blocks at each level, the budget is geometrically decreasing.
\end{proof}

In particular setting:
\[
\alpha = \frac{m}{6n \log_2(4 rd^2 \log(rdL/\delta)  + |C^\star|)}
\]
the budget for the adaptive phase is $\le m/2$. 

\subsection{The Passive Phase}
\label{sec:passive}
In the passive phase, we need to compute two key quantities, (1) the energy of
$\mathbf{1}_{C^\star}$ that remains in the span of $\Kcal$ and (2) the
estimation error of the projection that we perform. Recall that the space we are
interested in is $U = \textrm{span}\{\mathbf{1}_D\}_{D \in \Kcal}$ and let $U$
be a basis for this subspace. Let $\hat{\Mcal}$ denote the maximal clusters
retained in the adaptive phase while $\Mcal$ denotes all of the maximal
clusters. Throughout this section let $\xb = \mu \mathbf{1}_{C^\star}$.

Since $\{\mathbf{1}_M\}_{M \in \hat{\Mcal}}$ is a subspace of $U$ we
  know that:
\[
||\Pcal_U \mathbf{1}_{C^\star}||^2 \ge \sum_{M \in \hat{\Mcal}} \left( \frac{|C^\star
  \cap M|}{\sqrt{|M|}}\right)^2 = \sum_{M \in \hat{\Mcal}} |M|
\]
which means that (using Lemma~\ref{lem:adaptive_keep}):
\begin{eqnarray*}
\mathbb{E} ||(I-\Pcal_U)\mathbf{1}_{C^\star}||^2 &\le& \mathbb{E}\sum_{M \notin
  \Kcal} |M| = \sum_{M \in \Mcal} |M| \mathbb{P}[M \notin \Kcal]\\ &\le &
\sum_{M \in \Mcal} |M| L \exp\{-1/2(\sqrt{\alpha} |M| \mu/\sigma - z)^2\}
\\ &\le &
|C^\star| L \exp\{-1/2 (\sqrt{\alpha} |M_{\min}| \mu/\sigma - z)^2\}
\end{eqnarray*}
Since $q$ is a constant, $z$ is also constant. If $\mu/\sigma >
2z/(|M_{\min}|\sqrt{\alpha})$ (this will be dominated by other restrictions on the SNR)
then this expression is bounded by:
\[
\le |C^\star| L \exp\{-1/8 \alpha |M_{\min}|^2 \mu^2/\sigma^2 \}
\]
Applying Markov's inequality, we have that with probability $\ge 1-\delta$:
\[
||(I - \Pcal_U)\mathbf{1}_{C^\star}||^2 \le \frac{|C^\star|L}{\delta} \exp\{-1/8
\alpha |M_{\min}|^2 \mu^2/\sigma^2\}
\]

Now we study the passive sampling scheme. If $\yb =
\sqrt{\beta}U^T\xb + \mathbf{\epsilon}$ where $\mathbf{\epsilon} \sim
\mathcal{N}(0, \sigma^2 I_{|\Kcal|})$ then:
\[
\hat{\xb} = \sqrt{1/\beta} U\yb = \Pcal_U\xb + \sqrt{1/\beta}
U\mathbf{\epsilon}
\]
So that:
\[
||\hat{\xb} - \Pcal_U \xb||^2 = \frac{1}{\beta}
||U\mathbf{\epsilon}||^2 = \frac{1}{\beta} ||z||^2
\]
where $z \sim \mathcal{N}(0, \sigma^2 I_{|\Kcal|})$ is a $|\Kcal|$-dimensional
Gaussian vector. Concentration results for Gaussian vectors (or Chi-squared
distributions) show that there is a constant $c$ such that for $n$ large enough
$||z||^2 \le c \sigma^2 |\Kcal|$ with probability $\ge 1-\delta$.

Putting these two bounds together gives us a high probability bound on the
squared error (note that the cross term is zero since $\hat{\xb} \in U$):
\begin{eqnarray*}
||\hat{\xb} -\xb||^2 &\le& ||\hat{\xb} - \Pcal_U\xb||^2 + ||(I -
\Pcal_U)\xb||^2\\ & \le & \frac{c \sigma^2 |\Kcal|}{\beta} + \frac{|C^\star| L
\mu^2}{\delta}\exp\{-1/8 \alpha |M_{\min}|^2 \mu^2/\sigma^2\}
\end{eqnarray*}

\subsection{Recovering $C^\star$}

The error guarantee of the optimization phase is based on the following lemma:
\begin{lemma}
Let $\hat{C}$ denote the solution to:
\[
\textrm{argmax}_{C \subset [n]} \frac{\hat{\xb}^T \mathbf{1}_C}{||\hat{\xb}||\sqrt{|C|}}
\]
then if $4 ||\xb - \hat{\xb}||^2 < \mu^2 |C^\star|$:
\[
d(\hat{C}, C^\star) \triangleq 1 - \frac{|\hat{C} \cap C^\star|}{\sqrt{|\hat{C}| |C^\star|}} \le \frac{4 ||\hat{\xb} - \xb||_2^2}{\mu^2 |C^\star|}
\]
\label{lem:recovery}
\end{lemma}
\begin{proof}
It is immediate from the definition of the $d$-distance that:
\[
d(\hat{C}, C^\star) \triangleq 1 - \frac{|\hat{C} \cap C^\star|}{\sqrt{|\hat{C}| |C^\star|}} = \frac{1}{2}||\frac{\mathbf{1}_{\hat{C}}}{\sqrt{|\hat{C}|}} - \frac{\mathbf{1}_{C^\star}}{\sqrt{|C^\star|}}||^2
\]
By virtue of the fact that $\hat{C}$ solves the optimization problem, we also know:
\[
||\frac{\mathbf{1}_{\hat{C}}}{\sqrt{|\hat{C}|}} - \frac{\hat{\xb}}{||\hat{\xb}||}||^2 = 2 - 2\frac{\hat{\xb}^T \mathbf{1}_{\hat{C}}}{||\hat{\xb}|| \sqrt{|\hat{C}|}} \le ||\frac{\mathbf{1}_{C^\star}}{\sqrt{|C^\star|}} - \frac{\hat{\xb}}{||\hat{\xb}||}||^2
\]
Which, when coupled with the first identity gives us:
\[
d(\hat{C}, C^\star) = \frac{1}{2}||\frac{\mathbf{1}_{\hat{C}}}{\sqrt{|\hat{C}|}} - \frac{\mathbf{1}_{C^\star}}{\sqrt{|C^\star|}}||^2 \le 2 ||\frac{\mathbf{1}_{C^\star}}{\sqrt{|C^\star|}} - \frac{\hat{\xb}}{||\hat{\xb}||}||^2
\]
This follows by adding and subtracting $\hat{\xb}/||\hat{\xb}||$ and applying the triangle inequality to $\sqrt{d(\hat{C}, C^\star)}$.
Call $\theta = \angle (\mathbf{1}_{C^\star}, \hat{\xb})$.
Then we have:
\begin{eqnarray*}
||\frac{\mathbf{1}_{C^\star}}{\sqrt{|C^\star|}} - \frac{\hat{\xb}}{||\hat{\xb}||}||^2 &=& 2 - 2\cos \theta\\
||\xb - \hat{\xb}||^2 &=& ||\xb||^2 + ||\hat{\xb}||^2 - 2||\xb|| ||\hat{\xb}|| \cos \theta \\
\frac{||\xb - \hat{\xb}||^2}{||\xb|| ||\hat{\xb}||} &=& \frac{||\xb||}{||\hat{\xb}||} + \frac{||\hat{\xb}||}{||\xb||} - 2 \cos \theta
\end{eqnarray*}
On the right hand side is an expression of the form $1/a + a$.
This is $\ge 2$ for $a \ge 0$ as it is in this case, and with this in mind we see that:
\[
d(\hat{C}, C^\star) \le 2\left(2 - 2\cos\theta\right) \le \frac{2||\xb - \hat{\xb}||^2}{||\xb|| ||\hat{\xb}||}
\]
Looking just at the denominator of the right hand side, we can lower bound by:
\[
||\xb|| ||\hat{\xb}|| \ge ||\xb||^2 - ||\xb|| ||\xb - \hat{\xb}|| \ge \frac{1}{2}\mu^2 |C^\star|
\]
Where in the first step we used the triangle inequality, and in the second we used that $||\xb|| = \mu \sqrt{|C^\star|}$ and the assumption whereby $4 ||\xb - \hat{\xb}||^2 \le \mu^2 \sqrt{|C^\star|}$.
Plugging this in to the bound on $d(\hat{C}, C^\star)$ concludes the proof.
\end{proof}

\subsection{Proof of Corollary~\ref{cor:approximate}}
The proof of the corollary parallels that of the main theorem. In the adaptive
phase, we instead show that with high probability we retain all clusters of size
$\ge t$ for some parameter $t$. Then since we are not interested in recovering
the smaller clusters, we can safely ignore the energy in $C^\star$ that is
orthogonal to $U$. This means that the approximation error term from the
previous proof can be ignored.

\begin{lemma}
With probability $\ge 1- \delta$ we retain all clusters of size $\ge t$ as
long as:
\[
\frac{\mu}{\sigma} \ge \frac{1}{t\sqrt{\alpha}}\left(z + \sqrt{2\log(\frac{Lk}{t\delta})}\right)
\]
\end{lemma}
\begin{proof}
As in the proof of Lemma~\ref{lem:adaptive_keep} we can proceed with a union
bound. For a single block of size $\ge t$:
\begin{eqnarray*}
\mathbb{P}[M \notin \Kcal] \le L \mathbb{P}[y_M < \sigma z] \le L \exp\{ -
(\sqrt{\alpha} t \mu/\sigma - z)^2/2\}
\end{eqnarray*}
There are at most $|C^\star|/t = k/t$ such maximal blocks so with a union bound, we arrive
at the claim.
\end{proof}

The results from the adaptive phase show that all of the sufficiently large
maximal clusters are retained in $\Kcal$. If we let $\tilde{C} = \bigcup_{M \in
  \Mcal | |M| \ge t} M$ then $||(I - \Pcal_U) \mathbf{1}_{\tilde{C}}||^2 = 0$
with probability $\ge 1-\delta$. Applying the results from passive phase, in
particular Lemma~\ref{lem:recovery} we have:
\begin{eqnarray*}
d(\hat{C}, \tilde{C}) \le \frac{4\sigma^2}{\mu^2k}\frac{2c|\Kcal|^2}{m}
\end{eqnarray*}

Plugging for $|\Kcal|$ using the same bound as before, and setting $\alpha$ as
we did before gives the corollary.